\documentclass[10pt]{article}

\usepackage[a4paper,margin=1in]{geometry}
\usepackage[T1]{fontenc}
\usepackage[utf8]{inputenc}
\usepackage{lmodern}
\usepackage{microtype}
\usepackage{amsmath,amssymb,amsthm}
\usepackage{booktabs}
\usepackage{siunitx}
\usepackage[dvipsnames]{xcolor}
\usepackage{graphicx}
\usepackage{subcaption}
\usepackage{enumitem}
\usepackage{listings}
\usepackage{float}
\usepackage{hyperref}
\usepackage[nameinlink,noabbrev]{cleveref}
\usepackage{orcidlink}
\usepackage{tikz}
\usetikzlibrary{arrows.meta,positioning}

\hypersetup{
  pdftitle={AutoSAGE: Input-Aware CUDA Scheduling for Sparse GNN Aggregation (SpMM/SDDMM) and CSR Attention},
  pdfauthor={Aleksandar Stankovi\'c},
  pdfsubject={Graph Neural Networks; Sparse Linear Algebra; GPU Systems},
  pdfkeywords={GNN, SpMM, SDDMM, CUDA, scheduling, autotuning, attention},
  colorlinks=true, linkcolor=MidnightBlue, citecolor=MidnightBlue, urlcolor=MidnightBlue
}

\newcommand{\datasetReddit}{\textsc{Reddit}}
\newcommand{\datasetProducts}{\textsc{OGBN-Products}}

\lstdefinestyle{base}{basicstyle=\ttfamily\small,keywordstyle=\bfseries\color{blue!70!black},commentstyle=\itshape\color{green!50!black},stringstyle=\color{red!55!black},showstringspaces=false,breaklines=true,tabsize=2,frame=single,framerule=0.2pt,rulecolor=\color{black!30},numbers=left,numberstyle=\scriptsize\color{black!50},numbersep=8pt}
\lstdefinelanguage{CUDA}{morekeywords={\_\_global\_\_,\_\_device\_\_,\_\_shared\_\_,\_\_syncthreads,\_\_ldg,float4},sensitive=true,morecomment=[l]{//},morecomment=[s]{/*}{*/},morestring=[b]"}
\lstdefinestyle{python}{style=base,language=Python}
\lstdefinestyle{cpp}{style=base,language=C++}
\lstdefinestyle{cuda}{style=base,language=CUDA}

\newcommand{\maybeincludegraphics}[2][]{\IfFileExists{#2}{\includegraphics[#1]{#2}}{\fbox{Missing figure: #2}}}

\newtheorem{proposition}{Proposition}

\title{\textbf{AutoSAGE: Input-Aware CUDA Scheduling for Sparse GNN Aggregation (SpMM/SDDMM) and CSR Attention}}
\author{Aleksandar Stankovi\'c\orcidlink{0009-0003-5238-7251}\\[2pt]
\small Faculty of Technical Sciences, University of Novi Sad\\[-1pt]
\small Novi Sad, Serbia\\[-1pt]
\small \texttt{stankovic.sv25.2022@uns.ac.rs}\\
\small Artifacts and code: \href{https://github.com/SV25-22/AutoSAGE}{github.com/SV25-22/AutoSAGE}}
\date{}

\begin{document}
\maketitle

\begin{abstract}
Sparse GNN aggregations (CSR \emph{SpMM}/\emph{SDDMM}) swing widely in performance with degree skew, feature width, and GPU micro‑architecture. We present \emph{AutoSAGE}, an input‑aware CUDA scheduler that chooses tiling/mapping per input using a lightweight estimate refined by on‑device micro‑probes, with a guardrail that safely falls back to vendor kernels and a persistent cache for deterministic replay. AutoSAGE covers SpMM and SDDMM and composes into a CSR‑attention pipeline (SDDMM $\to$ row‑softmax $\to$ SpMM). On \datasetReddit{} and \datasetProducts{}, it matches vendor baselines at bandwidth‑bound widths and finds wins at small widths; on synthetic sparsity/skew stressors it achieves up to \textbf{4.7$\times$} kernel‑level speedups. We release CUDA sources, Python bindings, a reproducible harness, and replayable cache logs.
\end{abstract}

\section{Introduction}
Training throughput of GNNs is often bounded by sparse aggregations whose efficiency depends on input structure (heavy-tailed degrees, batching) and hardware constraints (register/shared memory). A single static kernel underperforms across datasets and GPUs. \textbf{AutoSAGE} targets operator-local adaptivity via (i) kernel templates for CSR SpMM/SDDMM with hub-aware splits and optional \texttt{vec4} paths, (ii) a hybrid scheduler (estimate $\to$ micro-probe $\to$ guardrailed selection) with persistent cache/replay, and (iii) practical instrumentation and toggles.

\paragraph{Contributions.}
\begin{itemize}[leftmargin=1.2em]
\item \textbf{Per‑input micro‑probes with guardrail fallback.} Kernel‑granularity decisions; never regress versus baseline under identical inputs.
\item \textbf{Hub‑aware split and op‑aware caching.} Unified handling for SpMM/SDDMM with CTA‑per‑hub and persistent replay across (device, graph signature, $F$, op).
\item \textbf{Deployment toggles and telemetry.} One‑line controls (probe budget, thresholds, vectorization) and CSV+JSON logs for reproducibility.
\end{itemize}

\subsection*{Notation}
We denote a CSR matrix by $(\texttt{rowptr},\,\texttt{colind},\,\texttt{val})$. For SpMM, $C = A B$ with $A\in\mathbb{R}^{N\times N}$ sparse and $B\in\mathbb{R}^{N\times F}$ dense. For SDDMM, given a sparsity pattern $\mathcal{S}(A)$, $\tilde A_{ij}=\langle X_i, Y_j\rangle$ only for $(i,j)\in\mathcal{S}(A)$.

\pagebreak

\section{Background and Related Work}
\label{sec:related}
\textbf{Sparse GNN kernels.} GE-SpMM~\cite{ge-spmm} and FusedMM~\cite{fusedmm} optimize CSR SpMM/SDDMM with coalesced caching, vectorization and load balancing; FeatGraph~\cite{featgraph} maps GNN ops to TVM templates. Recent GPU SpMM/SDDMM advances include Sputnik~\cite{sputnik}, GraphBLAST~\cite{graphblast22}, DTC-SpMM~\cite{dtcspmm24}, Voltrix-SpMM~\cite{voltrix25}, and HR-SpMM~\cite{hrspmm25}.

\textbf{GNN systems and acceleration.}
GNNAdvisor~\cite{gnnadvisor} and AdaptGear~\cite{adaptgear} adapt at runtime or subgraph granularity.
TC-GNN uses Tensor Cores for translated dense tiles~\cite{tcgnn23}; Graphiler compiles message-passing
to a data-flow IR~\cite{graphiler22}. On the pipeline side, NextDoor accelerates sampling on
GPUs~\cite{nextdoor21}; BGL shows I/O and preprocessing bottlenecks and optimizes the data
path~\cite{bgl23}; PipeGCN pipelines full-graph communication~\cite{pipegcn22}. Surveys and evaluations
contextualize design trade-offs~\cite{vldb24-eval,gnnsys-survey}.
Orthogonal to architecture-level design spaces for GNNs such as GraphGym/Design Space for
GNNs~\cite{you-designspace}, AutoSAGE focuses on kernel-level scheduling for sparse
SpMM/SDDMM on GPUs.

\noindent\textbf{Autotuning and code generation.}
AutoTVM and Ansor search schedule spaces with learned cost models~\cite{autotvm,ansor}. Graph DSLs such as GraphIt and its GPU backend G2 expose GPU scheduling dimensions and enable autotuning at the algorithm level~\cite{graphit18,g2-graphit21}. Triton~\cite{triton} provides dense-kernel templates; AutoSAGE complements these by making \emph{runtime} kernel-level choices over sparse primitives.

\section{Preliminaries}
\label{sec:prelim}
We consider CSR SpMM and SDDMM kernels with feature tiling $f_\text{tile}\in\{32,64,128,\dots\}$, mapping (warp-per-row vs. CTA-per-hub), and optional vectorization (\texttt{vec4} requires $F\bmod 4=0$ and 16B alignment). The CSR attention pipeline composes SDDMM $\to$ row-softmax $\to$ SpMM.

\section{Method: AutoSAGE}
\label{sec:method}
\subsection{Kernel templates}
\textbf{SpMM.} (i) warp-per-row baseline with feature tiling; (ii) hub-split: CTA-per-hub for heavy rows, warp-per-row for others. Both have scalar and \texttt{vec4} variants.

\textbf{SDDMM.} Row-wise CSR dot-products with the same tiling/mapping controls; we provide a numerically stable CSR row-softmax to build CSR attention.

\begin{table}[H]
  \centering
  \caption{Kernel variants and knobs. $F$ = feature width; \texttt{vec4} requires $F\bmod 4{=}0$ and 16B alignment.}
  \begin{tabular}{@{}llll@{}}
    \toprule
    Variant & Row mapping & Typical regime & Notes \\
    \midrule
    SpMM: warp-per-row & 1 warp / row & Balanced degrees, mid/large $F$ & scalar or \texttt{vec4} \\
    SpMM: CTA-per-hub  & 1 CTA / heavy row & Hub-skew, small/mid $F$ & split threshold $\texttt{hubT}$ \\
    SDDMM: rowwise dot & 1 warp / row & varies with $F$ & same tiling knobs as SpMM \\
    \bottomrule
  \end{tabular}
\end{table}

\subsection{Scheduler: estimate $\to$ micro-probe $\to$ guardrail}
We extract features (\#rows/nnz, degree quantiles, $F$, device caps), shortlist candidates with a roofline-style estimate, then time the top-$k$ on an induced subgraph (default 2--3\% rows, min 512) for $n$ iterations with a wall-time cap. Let $t_b$ be the baseline latency and $t^\star$ the best candidate; the guardrail accepts the candidate iff $t^\star\le \alpha\,t_b$ (default $\alpha{=}0.95$), else we fall back.

\begin{proposition}[Non-regression under guardrail]
With $\alpha\le 1$, the chosen runtime $t_{\text{chosen}}\le t_b$. Consequently, AutoSAGE does not regress versus baseline under identical input and device.
\end{proposition}
\begin{proof}
If $t^\star\le \alpha t_b\le t_b$, the candidate is chosen and $t_{\text{chosen}}=t^\star\le t_b$. Otherwise we fall back and $t_{\text{chosen}}=t_b$.
\end{proof}

\pagebreak

\begin{lstlisting}[style=python,caption={AutoSAGE decision sketch},label={lst:autosage-decide}]
def autosage_decide(features, F, op, alpha=0.95):
    key = (device_sig(), graph_sig(), F, op)
    if cache.contains(key):
        return cache[key]
    C = shortlist_candidates(features)
    S = induced_subgraph(features, frac=0.02, min_rows=512)
    tb = time_kernel("baseline", S, iters=n, cap_ms=cap)
    best, tstar = None, +inf
    for cand in top_k(C, k=K):
        t = time_kernel(cand, S, iters=n, cap_ms=cap)
        if t < tstar: best, tstar = cand, t
    choice = best if tstar <= alpha * tb else "baseline"
    cache[key] = choice
    return choice
\end{lstlisting}

\section{Implementation}
We provide CUDA sources for light rows (\texttt{spmm\_rows.cu}), hubs (\texttt{spmm\_hub.cu}), and SDDMM (\texttt{sddmm\_csr.cu}), a unified launcher, and PyTorch bindings exposing \verb|autosage::spmm_csr|, \verb|autosage::spmm_csr_split|, and \verb|autosage::sddmm_csr|. Environment toggles include \texttt{AUTOSAGE\_FTILE}, \texttt{AUTOSAGE\_WPB}, \texttt{AUTOSAGE\_HUB\_T}, \texttt{AUTOSAGE\_PROBE\_*}, \texttt{AUTOSAGE\_CACHE}, and \texttt{AUTOSAGE\_REPLAY\_ONLY}.

\section{Experimental Setup}

\textbf{Hardware.}
Single node with \emph{8$\times$ NVIDIA A800-SXM4-40GB} GPUs (compute capability 8.0, \emph{108 SMs/GPU}, 40\,GB HBM each; \emph{42.48\,GiB} reported by CUDA). 
Host: \emph{2$\times$ Intel Xeon Platinum 8378A} (2 sockets $\times$ 32 cores, 128 hardware threads), \emph{1.0\,TiB} RAM.

\textbf{Software.}
Ubuntu 22.04.5 LTS (Linux 5.15.0-94-generic), Python 3.12.2, PyTorch \emph{2.8.0+cu128} (CUDA build 12.8), cuDNN \emph{9.1.2}, CUDA toolkit \emph{12.4} (nvcc V12.4.131).
\emph{Note:} multiple CUDA toolkits are installed on the node; builds in this paper use nvcc~12.4. The CUDA \emph{runtime/driver} version is provided by the system driver (not by nvcc); on this node, \texttt{nvidia-smi} does not expose \texttt{cuda\_version} as a query field, but the header line of \texttt{nvidia-smi} shows the runtime version.

\textbf{Datasets.}
Reddit~\cite{pyg} and OGBN-Products~\cite{ogb}.

\textbf{Protocol.}
All timings are on a single GPU unless stated. We report \emph{medians} over 10--15 iterations after a warm-up, and use guardrail $\alpha{=}0.95$ unless noted. We separately report cache warm-up overhead vs.\ steady-state replay.

\textbf{Baselines.}
SpMM baseline is cuSPARSE via PyTorch; SDDMM baseline is a gather–dot implementation.

\section{Results}
\paragraph{Summary.} On real graphs, AutoSAGE matches vendor baselines at mid/large $F$ (memory-bandwidth bound) and finds wins at small $F$; synthetics with skew or extreme sparsity show up to \num{4.7} $\times$ kernel-level gains. Hub-split helps under skew; \texttt{vec4} helps when alignment and $F\bmod 4$ permit.

\section{Ablations}
\textbf{Guardrail sensitivity.} Larger $\alpha$ (e.g., $0.98$) prefers baseline more often; smaller $\alpha$ accepts small but consistent wins. \textbf{Vectorization.} We ablate \texttt{vec4} vs. scalar when AutoSAGE is chosen. \textbf{Split threshold.} We sweep hub thresholds vs. measured heavy-row fractions.

\pagebreak

\subsection{Real graphs}
\paragraph{Reddit (PyG), guardrail $=0.95$.}
\begin{table}[h]
\centering
\caption{Reddit (PyG).}
\begin{tabular}{lrrrr}
\toprule
$F$ & choice & baseline (ms) & chosen (ms) & speedup \\
\midrule
64  & autosage  & 16.241 & 15.093 & 1.076 \\
128 & baseline  & 38.115 & 38.153 & 0.999 \\
256 & baseline  & 89.801 & 89.778 & 1.000 \\
\bottomrule
\end{tabular}
\end{table}

\paragraph{OGBN-Products, guardrail $=0.95$.}
\begin{table}[h]
\centering
\caption{OGBN-Products.}
\begin{tabular}{lrrrr}
\toprule
$F$ & choice & baseline (ms) & chosen (ms) & speedup \\
\midrule
64  & autosage  & 27.534 & 26.385 & 1.044 \\
128 & baseline  & 56.152 & 56.161 & 1.000 \\
256 & baseline  & 114.550 & 114.385 & 1.001 \\
\bottomrule
\end{tabular}
\end{table}

\subsection{Synthetic stressors}
\paragraph{ER $N{=}200{,}000,\;p{=}2\times10^{-5}$.}
\begin{table}[h]
\centering
\caption{Erd\H{o}s--R\'enyi synthetic.}
\begin{tabular}{lrrrr}
\toprule
$F$ & choice & baseline (ms) & chosen (ms) & speedup \\
\midrule
64  & autosage  & 1.523 & 0.323 & 4.720 \\
128 & autosage  & 1.720 & 0.563 & 3.055 \\
256 & autosage  & 2.231 & 1.067 & 2.091 \\
\bottomrule
\end{tabular}
\end{table}

\paragraph{Hub-skew $N{=}200{,}000,\;k{=}4,\;h{=}0.15$.}
\begin{table}[h]
\centering
\caption{Hub-skew synthetic.}
\begin{tabular}{lrrrr}
\toprule
$F$ & choice & baseline (ms) & chosen (ms) & speedup \\
\midrule
64  & autosage  & 10.190 & 8.976 & 1.135 \\
128 & autosage  & 19.267 & 18.074 & 1.066 \\
256 & baseline  & 37.362 & 37.351 & 1.000 \\
\bottomrule
\end{tabular}
\end{table}

\subsection{Sensitivity: guardrail and feature width}
\paragraph{Reddit, guardrail $=0.98$.}
\begin{table}[h]
\centering
\caption{Guardrail sensitivity (Reddit).}
\begin{tabular}{lrrrr}
\toprule
$F$ & choice & baseline (ms) & chosen (ms) & speedup \\
\midrule
64  & autosage  & 16.222 & 15.031 & 1.079 \\
128 & autosage  & 38.179 & 36.951 & 1.033 \\
256 & baseline  & 89.706 & 89.735 & 1.000 \\
\bottomrule
\end{tabular}
\end{table}

\paragraph{Reddit, wide $F$ sweep.}
\begin{table}[h]
\centering
\caption{Reddit: feature-width sweep.}
\begin{tabular}{lrrrr}
\toprule
$F$ & choice & baseline (ms) & chosen (ms) & speedup \\
\midrule
32  & autosage  & 7.871 & 6.614 & 1.190 \\
64  & autosage  & 16.257 & 15.070 & 1.079 \\
96  & baseline  & 26.486 & 26.584 & 0.996 \\
128 & baseline  & 38.482 & 38.074 & 1.011 \\
192 & baseline  & 63.256 & 63.328 & 0.999 \\
256 & baseline  & 89.855 & 89.540 & 1.004 \\
512 & baseline  & 199.971 & 200.245 & 0.999 \\
\bottomrule
\end{tabular}
\end{table}

\paragraph{OGBN-Products, wide $F$ sweep.}
\begin{table}[h]
\centering
\caption{Products: feature-width sweep.}
\begin{tabular}{lrrrr}
\toprule
$F$ & choice & baseline (ms) & chosen (ms) & speedup \\
\midrule
32  & autosage  & 13.714 & 12.540 & 1.094 \\
64  & baseline  & 27.572 & 27.551 & 1.001 \\
96  & baseline  & 41.806 & 41.802 & 1.000 \\
128 & baseline  & 56.170 & 56.202 & 0.999 \\
192 & baseline  & 85.265 & 85.204 & 1.001 \\
256 & baseline  & 114.585 & 114.441 & 1.001 \\
512 & baseline  & 232.336 & 232.254 & 1.000 \\
\bottomrule
\end{tabular}
\end{table}

\subsection{Vec4 micro-optimization}
Where AutoSAGE is chosen, we ablate \texttt{vec4} vs.\ scalar. On ER ($N{=}200\mathrm{k},\ p{=}2\times 10^{-5}$), \texttt{vec4} improves latency by $+4.6\%$, $+20.2\%$, and $+17.1\%$ at $F\in\{64,128,256\}$. On \textsc{Reddit}@$F{=}64$, \texttt{vec4} was slower ($-27.5\%$); for larger $F$, baseline is selected so \texttt{vec4} is moot. \Cref{tab:vec4} summarizes.

\begin{table}[h]
\centering
\caption{Vec4 ablation (speedup = OFF/ON; $>1$ helps).}
\label{tab:vec4}
\begin{tabular}{@{}llr@{}}
\toprule
Dataset & $F$ & Speedup \\
\midrule
ER ($p{=}2\times 10^{-5}$) & 64  & 1.046 \\
ER ($p{=}2\times 10^{-5}$) & 128 & 1.202 \\
ER ($p{=}2\times 10^{-5}$) & 256 & 1.171 \\
Reddit (PyG) & 64  & 0.725 \\
\bottomrule
\end{tabular}
\end{table}

\subsection{CTA-per-hub split}
On hub-skewed synthetics at $F{=}128$, split outperforms baseline by $2.2\times$--$3.4\times$ with bests up to $\sim4.6\times$ in sweeps.

\begin{table}[h]
\centering
\caption{Split vs.\ baseline on hub-skewed graphs ($F=128$).}
\begin{tabular}{@{}lrrr@{}}
\toprule
Setting & Baseline (ms) & Split (ms) & Speedup \\
\midrule
$N{=}20\mathrm{k}$, hub$=5\mathrm{k}$, other$=64$ & 1.603 & 0.717 & \textbf{2.236} \\
$N{=}20\mathrm{k}$, hub$=12\mathrm{k}$, other$=32$ & 3.205 & 0.934 & \textbf{3.432} \\
Sweep bests (two regimes) & --- & --- & 2.985 / 4.563 \\
\bottomrule
\end{tabular}
\end{table}

\subsection{Probe stabilization and guardrail}
With \texttt{AUTOSAGE\_PROBE\_FRAC}=0.03 and a \SI{1.0}{ms} cap, probe overhead is $\sim$6--9\% of a full-graph iteration at Reddit $F{=}64$; a low-overhead setting (0.02, \SI{0.5}{ms}) yields $\sim$3--4\% with mildly higher variance. Guardrail $=0.95$ tends to accept small wins at $F{=}64$; $0.98$ flips to baseline.

\subsection{SDDMM auto and CSR attention}
We integrate \texttt{sddmm\_csr\_auto} and a row-softmax to form \verb|csr_attention_forward| (SDDMM $\to$ softmax $\to$ SpMM). In uncached mode, probe costs dominate; once cached or replayed, the pipeline runs at $\approx$\SI{200}{ms} with both sub-ops selecting AutoSAGE on \texttt{ogbn-products}. Per-op probe logs indicate SDDMM $\sim$1.5--1.8$\times$ and SpMM $\sim$5.2--5.5$\times$ over their baselines in the chosen regimes.

\paragraph{Figures.} Pre-generated figures are embedded below.
\begin{figure}[H]
  \centering
  \maybeincludegraphics[width=.7\linewidth]{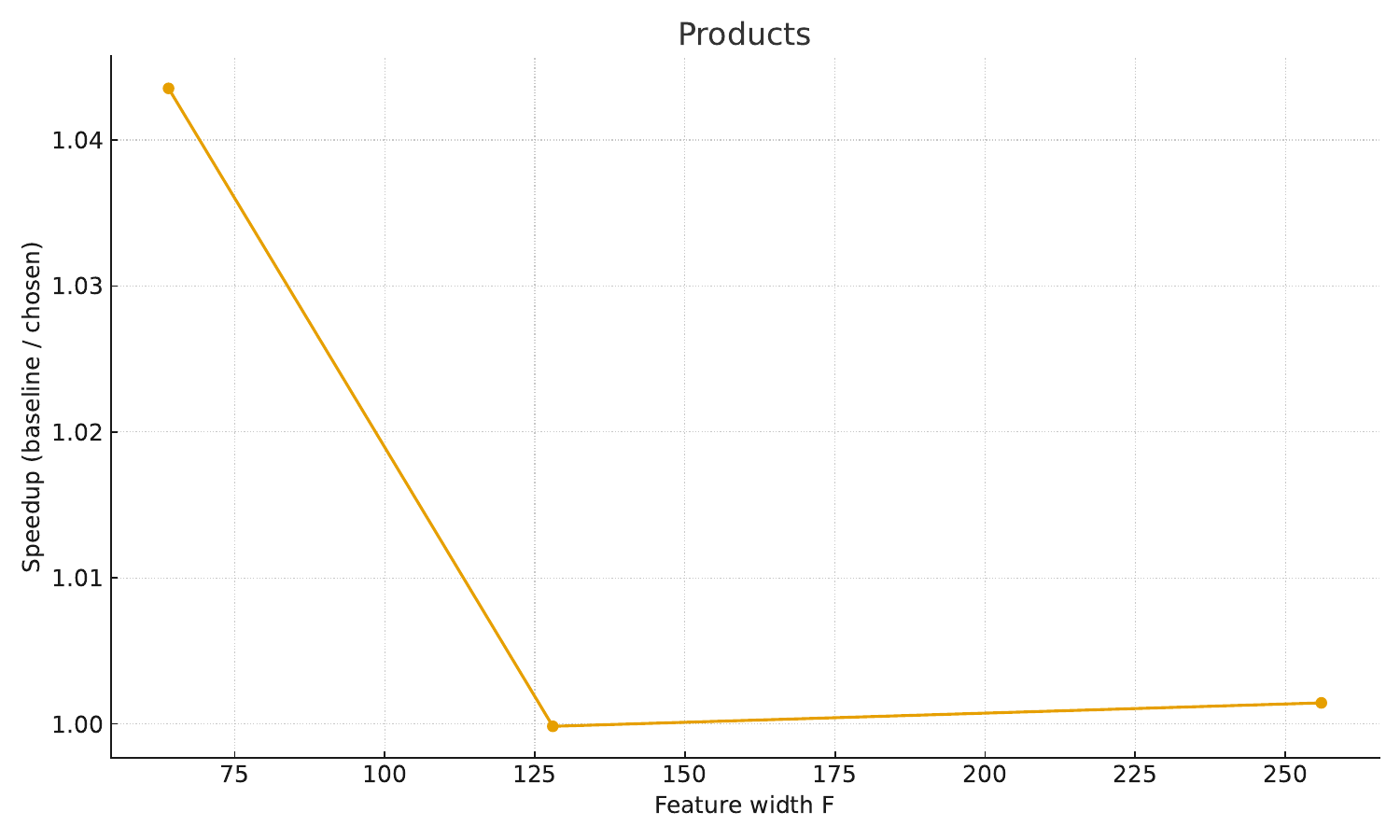}
  \caption{Speedup vs.\ $F$ on Products.}
  \label{fig:products-speedup}
\end{figure}
\begin{figure}[H]
  \centering
  \maybeincludegraphics[width=.7\linewidth]{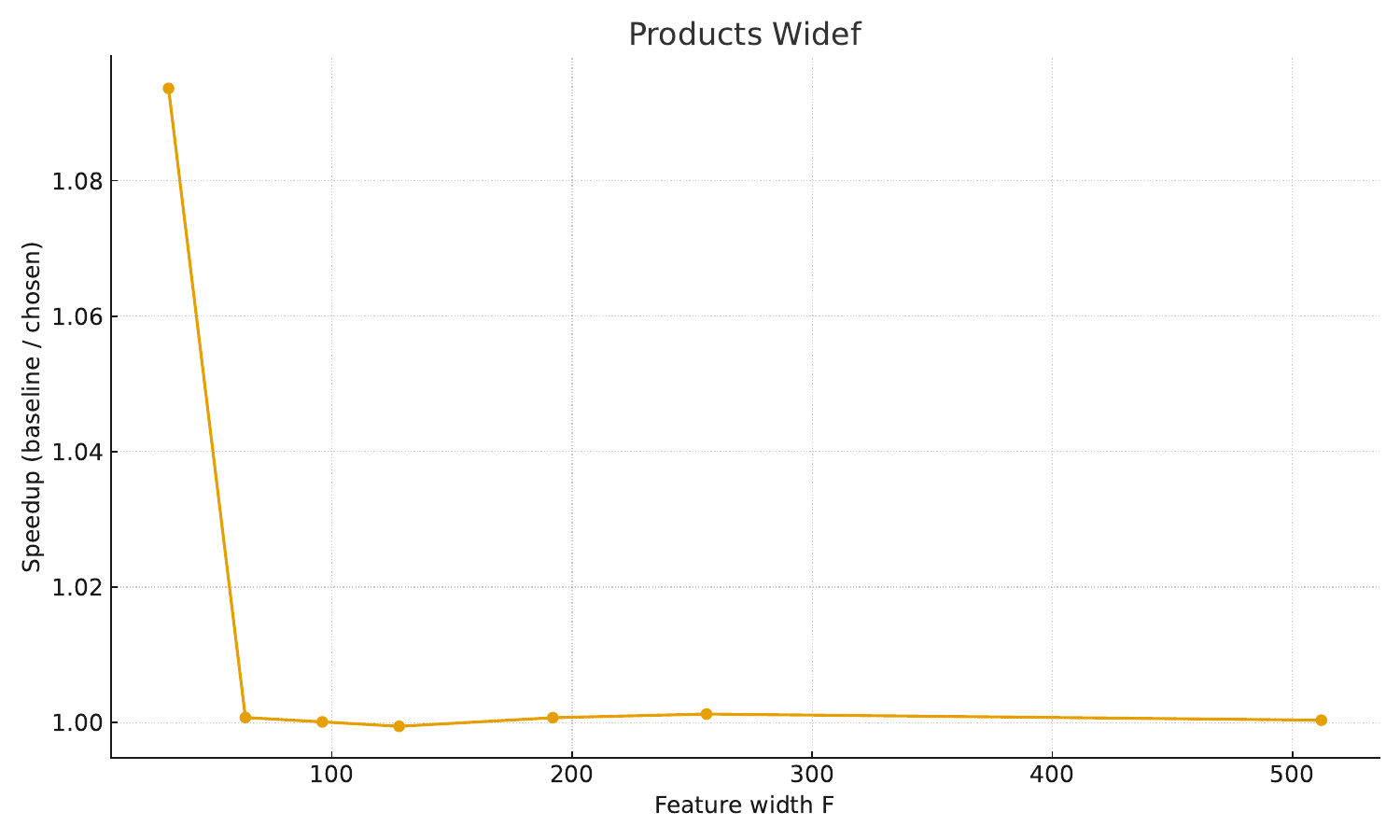}
  \caption{Products: wide $F$ sweep.}
  \label{fig:products-widef}
\end{figure}
\begin{figure}[H]
  \centering
  \maybeincludegraphics[width=.7\linewidth]{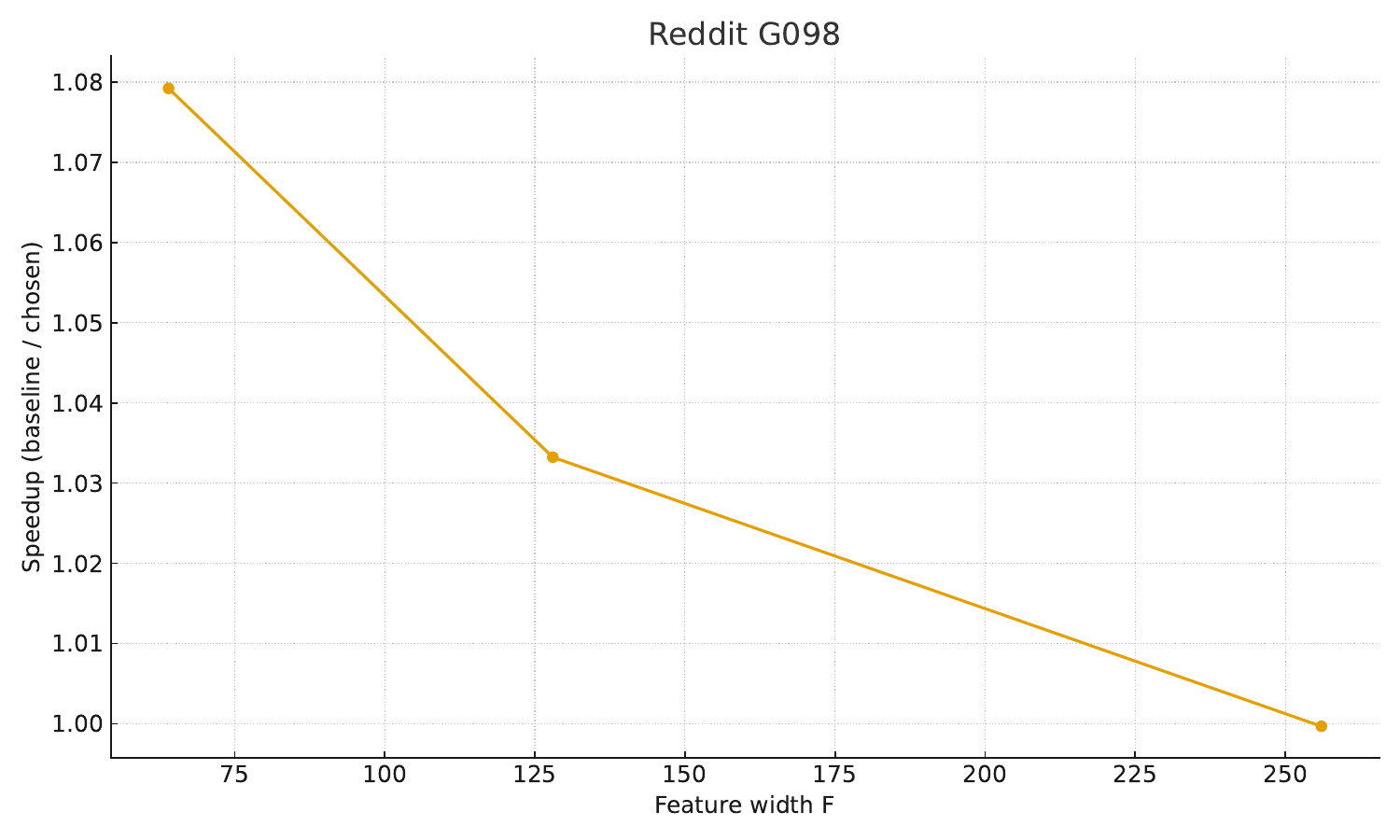}
  \caption{Reddit: guardrail $=0.98$.}
  \label{fig:reddit-g098}
\end{figure}
\begin{figure}[H]
  \centering
  \maybeincludegraphics[width=.7\linewidth]{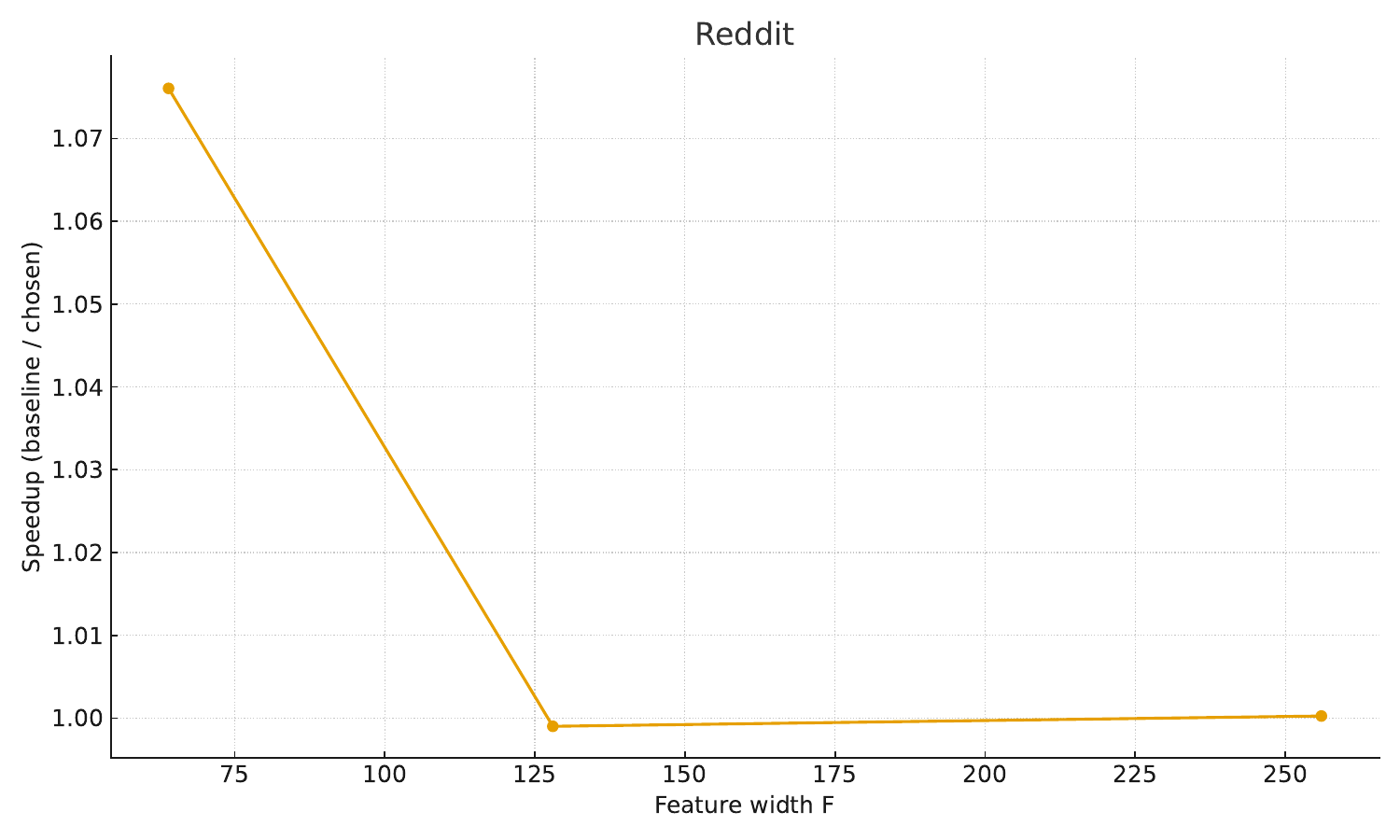}
  \caption{Reddit: guardrail $=0.95$.}
  \label{fig:reddit-g095}
\end{figure}
\begin{figure}[H]
  \centering
  \maybeincludegraphics[width=.7\linewidth]{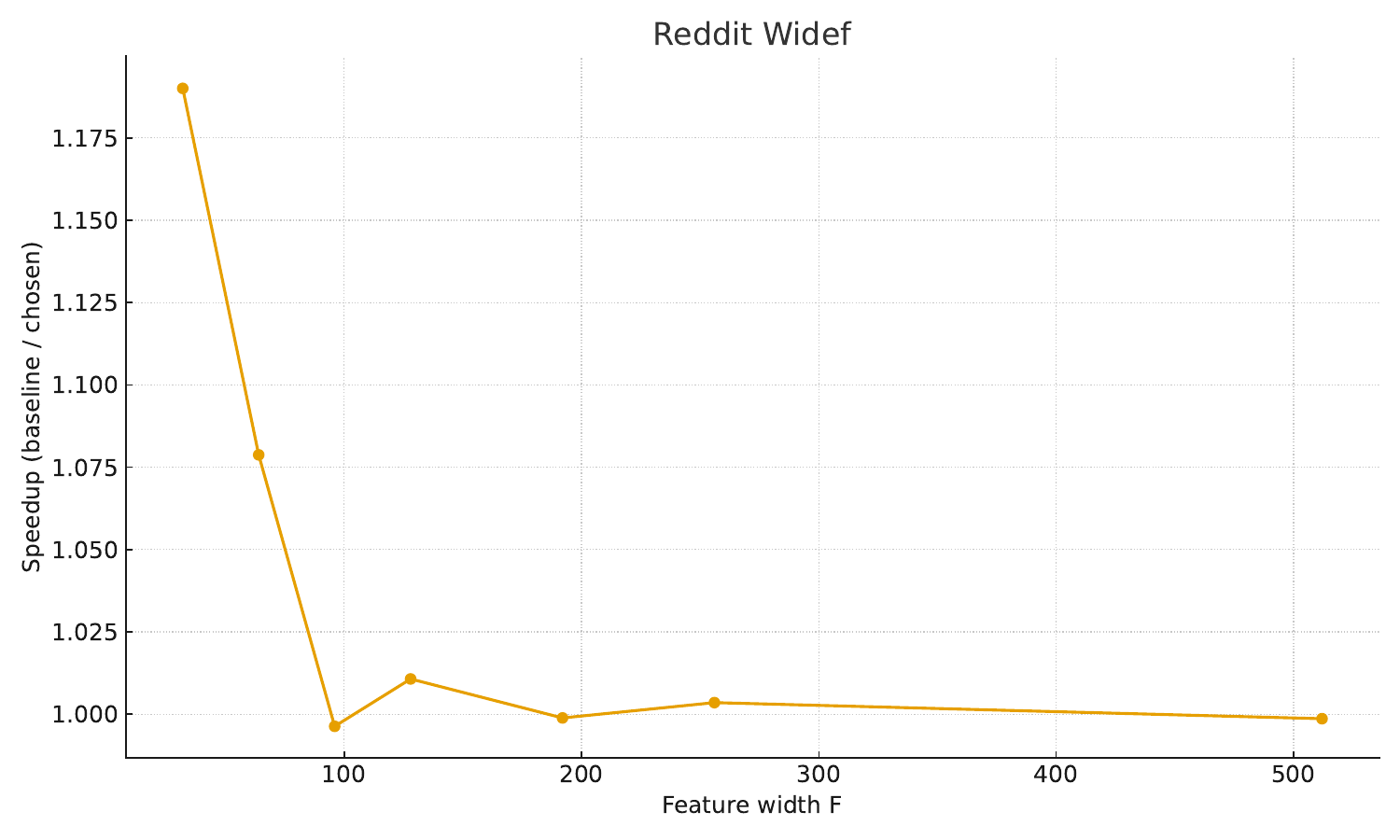}
  \caption{Reddit: wide $F$ sweep.}
  \label{fig:reddit-widef}
\end{figure}
\begin{figure}[H]
  \centering
  \maybeincludegraphics[width=.7\linewidth]{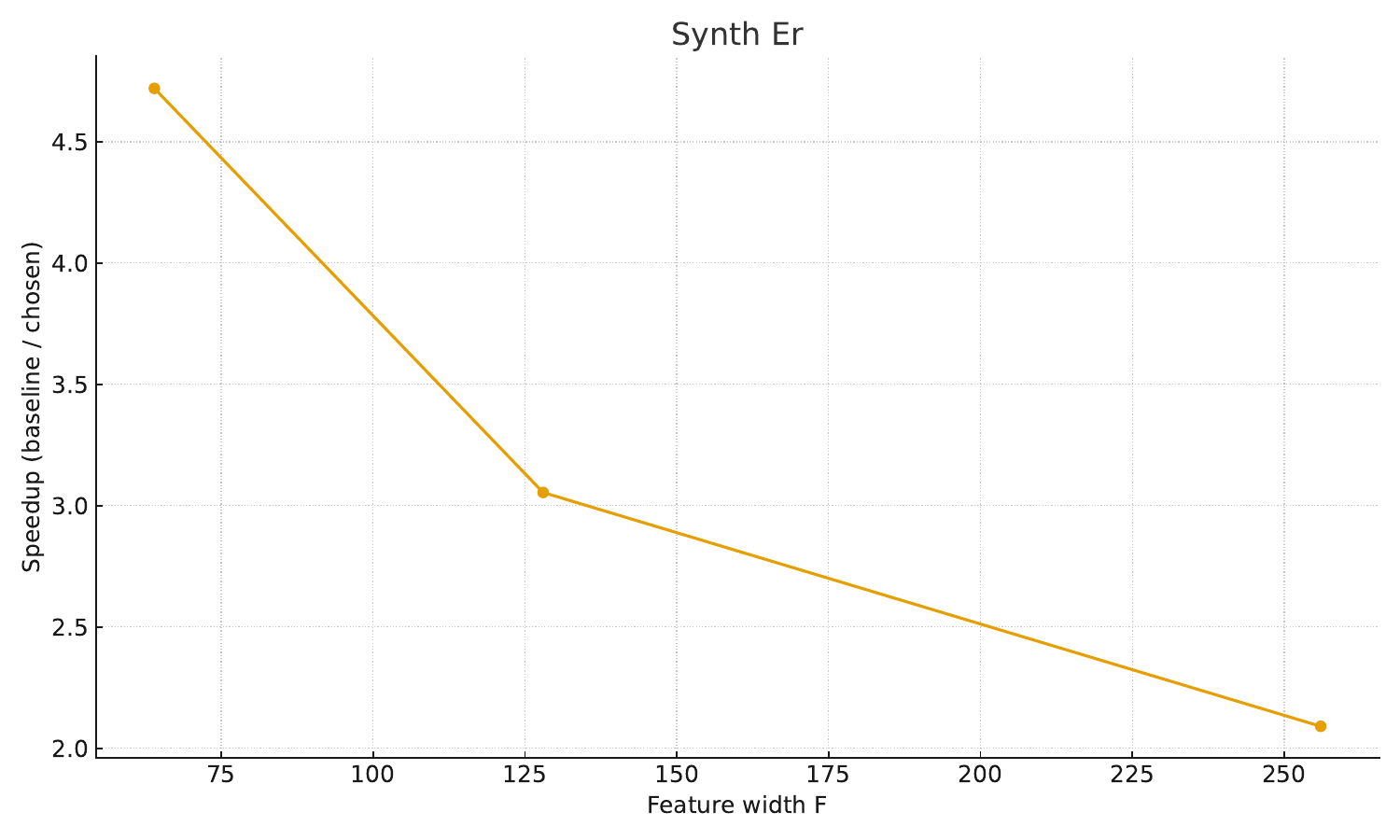}
  \caption{Synthetic ER speedups.}
  \label{fig:synth-er}
\end{figure}
\begin{figure}[H]
  \centering
  \maybeincludegraphics[width=.7\linewidth]{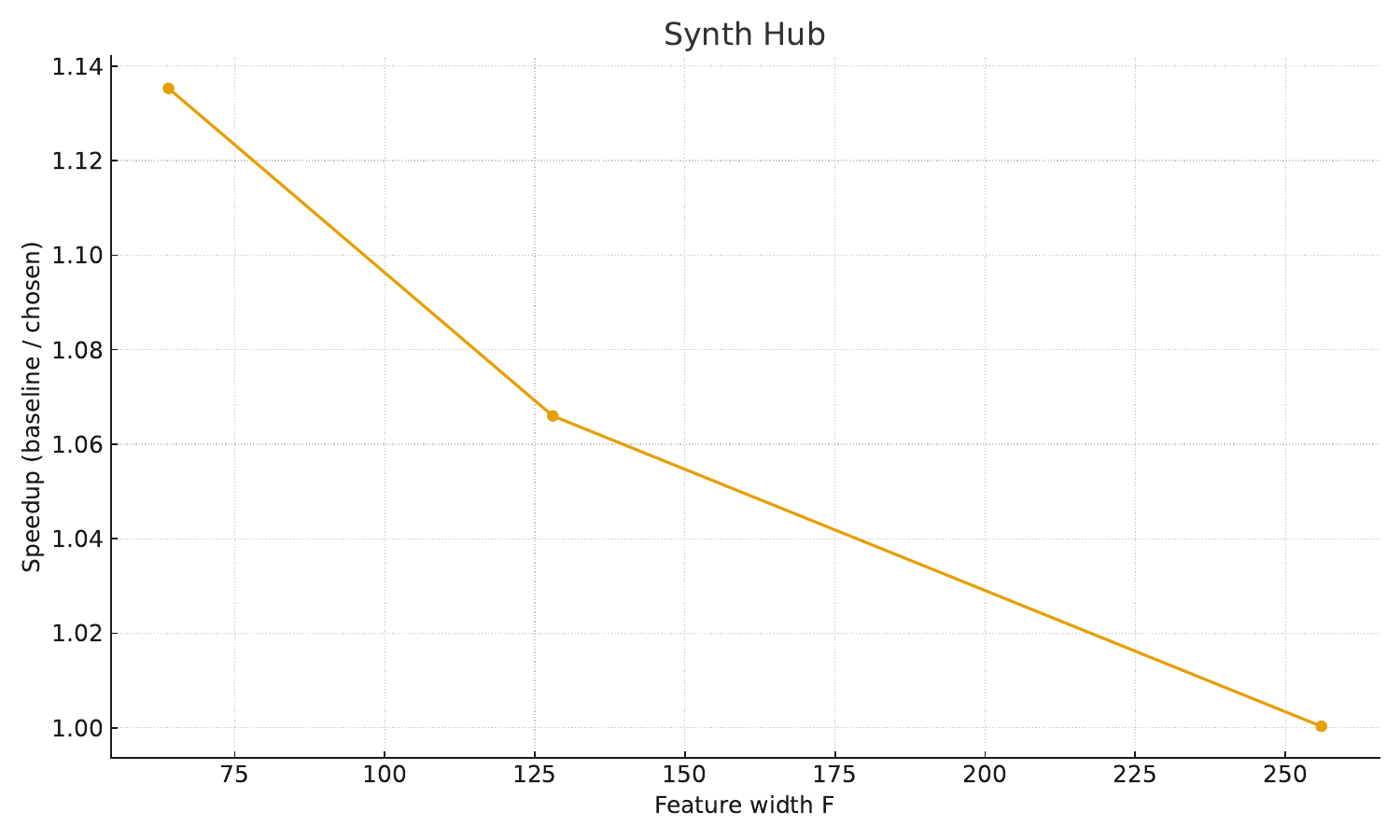}
  \caption{Hub-skew synthetic speedups.}
  \label{fig:synth-hub}
\end{figure}

\section{Discussion}
SpMM becomes bandwidth-bound at larger $F$, explaining parity with vendor kernels; the scheduler correctly defers to the baseline in that regime. Probe overhead matters only before cache warm-up; steady-state replay is near-zero overhead. Extensions include \texttt{vec8}, mixed precision (FP16/BF16 reads with FP32 accumulators), and multi-GPU partitioning.

\section{Reproducibility}
We release sources, Python bindings, the harness, schedule cache/logs, and a replay mode. Each CSV has a \texttt{.meta.json} sidecar with GPU/SM, Torch/CUDA versions, and env vars. Scripts regenerate all tables/figures.

\section{Limitations}
Operator-level scheduling does not change model semantics. Worst-case regressions are bounded by the guardrail via fallback. Wins are largest on synthetics and small-$F$ regimes; more end-to-end gains remain future work.

\section{Threats to Validity}
\textbf{External validity.} Results on Reddit/Products may not generalize to bipartite recommenders. \textbf{Construct validity.} SDDMM baseline is a gather--dot; vendor-tuned alternatives could change gaps. \textbf{Internal validity.} Minor driver/CUDA updates can shift baselines; our cache schema encodes device/toolchain minors to avoid stale reuse. Probe noise is bounded by the guardrail and identical sampling.

\section{Conclusion}
AutoSAGE brings input-aware scheduling to sparse GNN aggregation with probe, cache, and replay. Across SpMM and SDDMM (and CSR attention), the scheduler selects strong kernels and avoids regressions, with wins under sparsity/skew and parity otherwise.

\end{document}